\documentclass{article}

\usepackage[preprint]{neurips_2026}
\usepackage{amsmath}
\usepackage{amssymb}
\usepackage{mathtools}
\usepackage{amsthm}
\usepackage{graphicx}
\usepackage{subcaption}
\usepackage{tikz}
\theoremstyle{plain}
\newtheorem{theorem}{Theorem}[section]

\newtheorem{lemma}[theorem]{Lemma}
\newtheorem{example}[theorem]{Example}
\newtheorem{corollary}[theorem]{Corollary}
\theoremstyle{definition}
\newtheorem{definition}[theorem]{Definition}

\theoremstyle{remark}
\newtheorem{remark}[theorem]{Remark}

\usepackage[utf8]{inputenc} 
\usepackage[T1]{fontenc}    
\usepackage{hyperref}       
\usepackage{url}            
\usepackage{booktabs}       
\usepackage{amsfonts}       
\usepackage{nicefrac}       
\usepackage{microtype}      
\usepackage{xcolor}         
\usepackage{algorithm}
\usepackage{algpseudocode}
\usepackage{amsmath}
\usepackage{natbib}
\usepackage{comment}
\usepackage{wrapfig}

\title{HORST: Composing Optimizer Geometries \\ for Sparse Transformer Training}

\author{%
  Tom Jacobs \ \ \ \ Rohan Jain \ \ \ \ Rebekka Burkholz \\
  CISPA Helmholtz Center for Information Security\\
  \texttt{tom.jacobs@cispa.de} \\
}

\begin{document}

\maketitle

\begin{abstract}

   Sparsifying transformers remains a fundamental challenge, as standard optimizers fail to simultaneously encourage sparsity and maintain training stability. Effective adaptive optimizers exhibit an implicit $L_{\infty}$ bias favoring stability, yet, sparsity requires an $L_1$ bias. To integrate sparsity, we propose a composition of optimizer steps, which we cast as non-commutative operators to analyze and combine their optimization geometry in a principled way.
   This yields HORST (Hyperbolic Operator for Robust Sparse Training), a modular optimizer that inherits stability from adaptive methods while inducing $L_1$ sparsity bias through a hyperbolic mirror map. Our experiments demonstrate its utility for sparse training of transformers on both vision and language tasks. HORST consistently and significantly outperforms AdamW baselines across all sparsity levels, with large gains at higher sparsity.
\end{abstract}

\section{Introduction}

Transformers \citep{vaswani2017attention} have become a dominant architecture across deep learning, achieving state-of-the-art performance in areas such as natural language processing~\citep{brown2020gpt3} and computer vision~\citep{dosovitskiy2021vit}. Yet, they incur a steep computational cost: modern transformer models require billions of parameters and enormous memory footprints~\citep{Kaplan2020ScalingLF, Hoffmann2022TrainingCL}, making deployment at scale a practical challenge. Network sparsification, i.e. the removal of redundant weights, offers a promising route to model compression~\citep{lecun_1989_optimal, han2015learning}, reducing storage requirements and potentially enabling faster inference on hardware that supports sparse operations~\citep{gale2020sparsegpu, mishra2021accelerating}. While pruning methods have achieved considerable success in convolutional architectures \citep{peste2021acdc, jacobs2025hamhyperbolicstepregulate, evci-rigl}, they work less well for transformers: sparsification consistently leads to significant performance degradation especially at higher sparsity levels where compression would be most practically valuable. This contrast is not incidental. Convolutional networks are trained with standard SGD, while transformers rely on adaptive optimizers such as AdamW \citep{Loshchilov2017DecoupledWD} to train more stably. We argue that it is precisely this difference in optimization that underlies the difference in sparsifiability. The implicit bias induced by the optimizer during training shapes the geometry of the learned weights, and adaptive optimizers impose a bias that is fundamentally at odds with the structure required for sparse solutions \citep{warmuth2025how, jacobs2026never}.
Figure \ref{fig:placeholder} illustrates this through the standardized weight distributions of a pretrained ResNet-50~\citep{he2016resnet} (Convolutional) and a DeiT-base~\citep{touvron2021deit} (Transformer).

\begin{figure}
    \centering
    \includegraphics[width=0.4\linewidth]{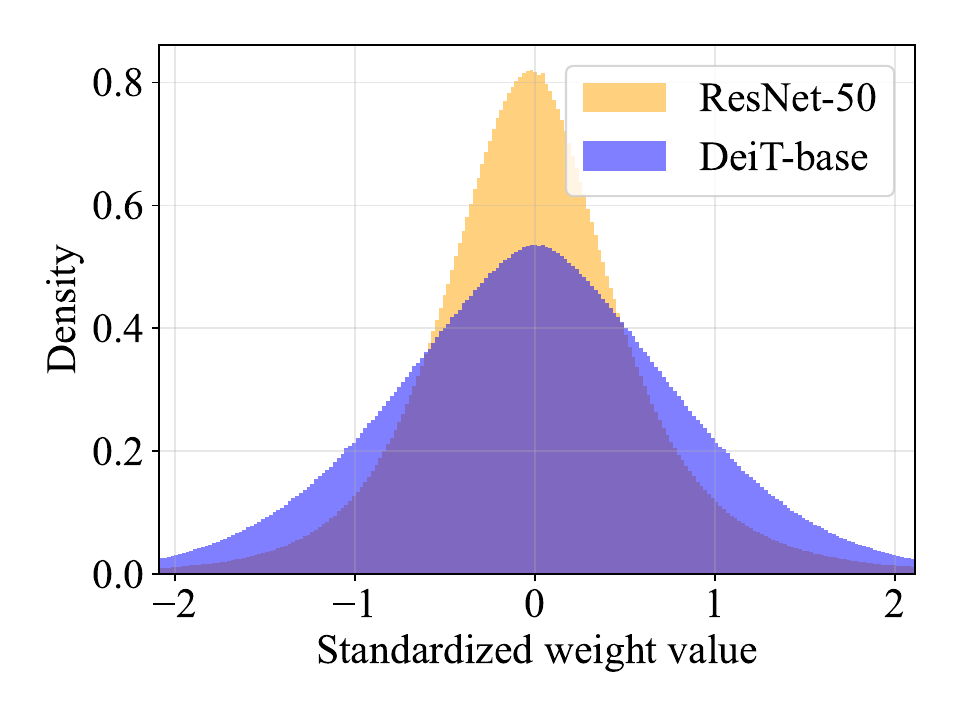}
    \qquad
    \includegraphics[width=0.4\textwidth]{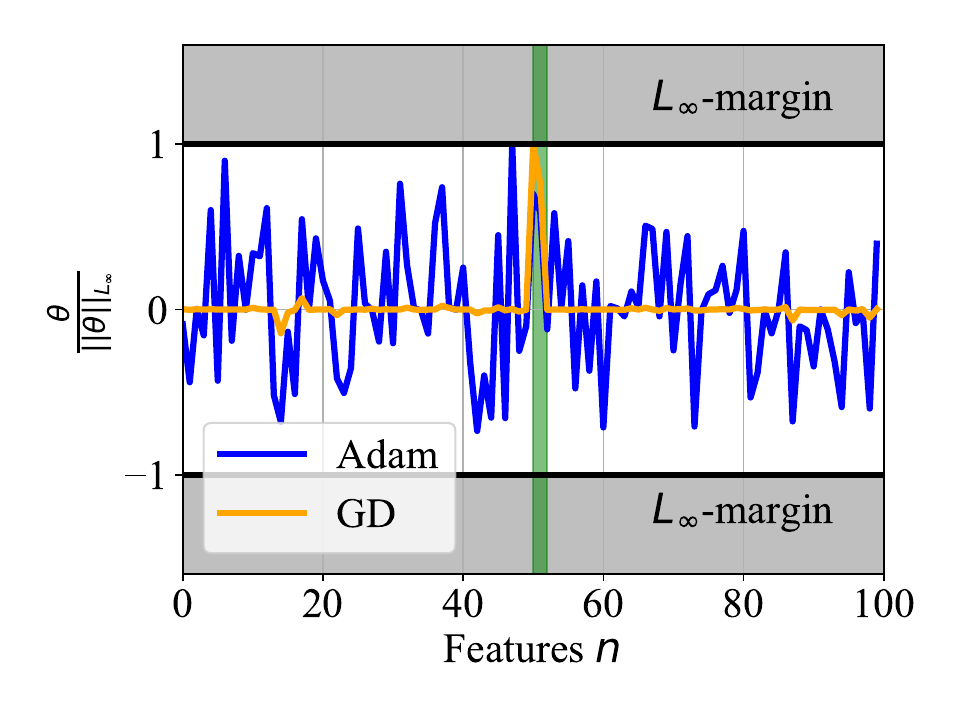}
    \caption{(Left) Standardized weight distributions for a pretrained ResNet-50 (with SGD) and a DeiT (with AdamW). The transformer weights are more spread out, which illustrates the different implicit bias ($L_2$ and $L_{\infty}$). (Right) Linear separable binary classification with a deep diagonal linear network $\langle u \odot v, x\rangle$ where $\theta = u \odot v$. GD learns the sparse features (green) while Adam overfits on the data.}
    \label{fig:placeholder}
\end{figure}

Beyond minimizing the training loss, optimizers impose a preference over which solution is selected when many equivalent ones exist. In the overparameterized regime, where infinitely many parameter configurations achieve zero training loss, this preference (i.e. the implicit bias of the optimizer) is influenced by the geometry of the update rule \citep{gunasekar2020characterizing}. Gradient descent, for instance, finds minimum $L_2$-norm solutions in simple settings such as linear regression or the $L_2$-max margin in binary classification with homogeneous neural networks \citep{Lyu2019GradientDM}. Adaptive optimizers such as Adam~\citep{kingma2017adammethodstochasticoptimization} are closely related to sign descent \citep{tsilivis2025flavors, jacobs2026never, bernstein2018signsgd}, a form of steepest descent with respect to the $L_{\infty}$-norm, 
and accordingly exhibit an $L_{\infty}$ implicit bias: weight magnitudes tend to equalize across parameters rather than concentrate. This is exactly the wrong bias for sparsity, which requires weight mass to concentrate on a small subset of parameters while the rest are driven to zero. However in the steepest descent family with respect to a general norm $|| \cdot ||$, inducing an 
$L_1$ bias would require taking steps with respect to the $L_1$-norm, which corresponds to coordinate descent, updating only the single largest gradient coordinate at each step, a procedure that is not practical for large deep networks as schematically illustrated in Figure \ref{fig: dichotomy}. 

A different optimizer geometry class, mirror descent, offers a different route: by replacing the Euclidean geometry of gradient descent with one induced by a mirror map $\nabla R : \mathbb{R}^n \rightarrow \mathbb{R}^n$, the parameters can also be made to favor sparse solutions \citep{sun2022mirror, jacobs2024maskmirrorimplicitsparsification}. 
In particular, the hyperbolic entropy mirror map is known to induce an $L_1$ bias, promoting sparsity \citep{woodworth2020kernel, pesme2021implicit}. This idea was recently exploited in the HAM optimizer \citep{jacobs2025hamhyperbolicstepregulate}, which alternates a standard optimizer step with a hyperbolic mirror step, and was shown to improve sparse training of convolutional networks. However, transformers present a fundamentally different challenge: they leverage a different geometric bias ($L_{\infty}$) for stable training. This is incompatible with standard mirror descent which directly uses the gradient. Therefore, an explanation for the tension between modern optimizers and sparsity is a geometric one. In this work, we propose to combine them in a principled way, leveraging stability of adaptive methods and the $L_1$ sparsity bias of mirror descent.

We formalize this tension by treating optimization steps as functional operators \citep{bernstein2025modular} and identifying a structural dichotomy: steepest descent is especially well suited to the $L_{\infty}$ bias that stabilizes adaptive training, while mirror descent is especially well suited to the $L_1$ bias needed for sparsity. Relying on either class alone 
therefore comes at a cost. We show that composing both operations in the right order resolves this conflict, yielding \textsc{HORST} (Hyperbolic Operator for Robust Sparse Training), a modular optimizer that integrates seamlessly with standard adaptive methods. We validate \textsc{HORST} on vision transformers and show consistent and significant improvements over AdamW baselines across sparsity levels, with the largest gains in the high-sparsity regime where adaptive optimizers struggle most.

\textbf{Contributions:}
\begin{itemize}
    \item We reveal a dichotomy between steepest descent and mirror descent geometries as shown in Figure \ref{fig: dichotomy}.
    \item We introduce optimizer-operator composition as a design principle in \S \ref{section:properties}.
    \item We show that the entropy mirror map can overwrite the steepest-descent implicit bias in Theorem \ref{theorem:implicit-bias}. This motivates the composed sparsity aware optimizer (Algorithm \ref{alg:cap}): Hyperbolic Operator for Robust Sparse Training (HORST).
    \item We experimentally evaluate on sparse training settings in vision and language tasks (\S \ref{section : experiments}).
\end{itemize}

\section{Related work}

\paragraph{Steepest descent and modern optimization.}
Recent work views optimizers as modular operations on groups of parameters \citep{bernstein2025modular}. We build on this and focus on another basic operation: composition.
This is motivated by the role of implicit bias in modern optimization, as Adaptive methods can overfit rotationally invariant sparse targets \citep{warmuth2025how}, and many large-scale optimizers can be viewed as steepest descent under suitable norms or geometries \citep{fan2025implicit,DBLP:journals/corr/abs-2405-14813, tsilivis2025flavors, jacobs2026never}.
This includes recent matrix and spectral methods such as Muon \citep{jordan2024muon}, and Scion\citep{pethick2025training}.
Adam and AdamW related to sign descent based methods \citep{Orvieto2025InSO} also induce distinctive implicit biases: both exhibit $L_\infty$ type of bias \citep{zhang2024the, xie2024implicit}. This motivates composing a stable adaptive update with a sparsity-promoting mirror update.

Moreover, related multiplicative updates appear in log-normal training dynamics \citep{nishida2025lognormalmultiplicativedynamicsstable} and in MADAM, whose exponential update is tied to the entropy mirror map $\nabla R(\theta)=\log(\theta)$ and preserves the initial parameter signs \citep{Bernstein2020LearningCF}. The closest related optimizer is HAM \citep{jacobs2025hamhyperbolicstepregulate}, which alternates a base optimizer step with the entropy mirror step to obtain the sparsity bias of this reparameterization improving state-of-the-art sparse training methods. 
HORST follows the same geometric principle, but in the transformer setting where the base optimizer relies on AdamW: by composing an adaptive steepest-type update with a hyperbolic mirror update.

\begin{figure}
\centering
\begin{tikzpicture}[
  box/.style={draw, minimum width=2.8cm, minimum height=1.8cm, align=center, line width=0.4pt},
  header/.style={draw, minimum width=2.8cm, minimum height=0.8cm, align=center, font=\bfseries\small, line width=0.4pt},
  rowlabel/.style={draw, minimum width=2cm, minimum height=1.8cm, align=center, font=\bfseries\small, line width=0.4pt},
  arrow/.style={->, thick, >=stealth},
]

\node[header, fill=gray!20] (h1) at (3.4, 0) {$L_\infty$ bias};
\node[header, fill=gray!20] (h2) at (6.2, 0) {$L_1$ bias};

\node[rowlabel, fill=gray!20] (r1) at (1, -1.3) {Steepest \\ descent};
\node[rowlabel, fill=gray!20] (r2) at (1, -3.1) {Mirror \\ descent};

\node[box, fill=green!35] (q1) at (3.4, -1.3) {Sign descent \\ ($\simeq$ Adam) \\[2pt]
  \textit{\footnotesize bounded, stable}};
\node[box, fill=red!35] (q2) at (6.2, -1.3) {Coordinate \\ descent \\[2pt]
  \textit{\footnotesize infeasible for} \\
  \textit{\footnotesize large networks}};
\node[box, fill=red!35] (q3) at (3.4, -3.1) {$\cosh$ \\ entropy \\\textit{\footnotesize loses coercivity,} \\
  \textit{\footnotesize stalls}};
\node[box, fill=green!35] (q4) at (6.2, -3.1) {Hyperbolic \\ entropy \\[2pt]
  \textit{\footnotesize unbounded, sparsity}};

\end{tikzpicture}
 \caption{\textbf{Steepest-Mirror Descent Dichotomy:}  Each geometric optimization class is effective at inducing the corresponding dual implicit bias. Both coordinate descent and $\cosh$-entropy are infeasible due to slow convergence.}
    \label{fig: dichotomy}
\end{figure}

\paragraph{Mirror descent and reparameterizations.}
Reparameterized gradient flows can be viewed as mirror flows \citep{Li2022ImplicitBO}, building on classical implicit-regularization results for matrix factorization and overparameterized models \citep{gunasekar2017implicit, arora2019implicitregularizationdeepmatrix, azulay2021implicitbiasinitializationshape, optimalsparse, Zhao_2022, li2021implicitsparseregularizationimpact, woodworth2020kernel}. 
For sparsity, the central example is the diagonal-linear reparameterization \citep{jacobs2024maskmirrorimplicitsparsification}, where the induced hyperbolic-entropy geometry interpolates between an $L_2$ and an $L_1$ bias, with weight decay and stochasticity driving the implicit bias towards $L_1$ \citep{jacobs2025mirror, pesme2021implicit}. 
This mirror-flow perspective has been extended to large learning rates, momentum, explicit regularization, and weight normalization \citep{Even2023SGDOD, PapazovPF24, jacobs2024maskmirrorimplicitsparsification, jacobs2025mirror, Chou2023RobustIR, chou2024induce}, and to implicit-bias characterizations for mirror descent and exact dynamics of deep linear networks \citep{sun2022mirror, pesme2024implicit, Kunin2024GetRQ, domin2024from}. 
Our proof uses this line of work in the following way: in the steepest mirror flow connection of \citet{jacobs2026never}, we interpret the composed steepest mirror dynamics as steepest descent under a multiplicative reparameterization, which yields the $L_1$ sparse-bias characterization in Theorem~\ref{theorem:implicit-bias}. 
HORST then turns this continuous-time mechanism into a discrete optimizer by composing an adaptive steepest-type step with an entropy mirror map update.

\paragraph{Classic sparse training.}
Sparse training has been studied extensively in vision tasks \citep{adnan2026sparseopt, mohammed2025sparsetraining, Hua2025CannistraciHebbTO, pham2023paths, pham2026the, Lunk2026SparseTO, pmlr-v202-gadhikar23a, Kolb2026}, where dynamic sparse training methods such as SET \citep{mocanu-evo-train} and RigL \citep{evci-rigl, struct-dst} alternate magnitude pruning with gradient-informed regrowth strategies, and the lottery ticket hypothesis \citep{frankle2019lottery, convexist, depthexist} established that there are trainable sparse subnetworks found at initialization. Among dense-sparse pipelines, AC/DC \citep{peste2021acdc} alternates dense and sparse phases through iterative hard thresholding and is the method we adopt as the sparsification backbone for our empirical evaluation. \cite{gadhikar2024masks, gadhikar2025signinlotteryreparameterizingsparse} further showed that the success of dense-sparse methods over pruning at initialization (PaI) traces back to their ability to flip parameter signs that PaI cannot. This is a secondary motivation we adopt for the alternating scheme in HORST as the multiplicative update on its own would prevent sign flips.

\paragraph{Sparsity in transformers.}
Transformer sparsification has been studied through several complementary approaches. Post-training pruning methods such as SparseGPT \citep{frantar2023sparsegpt}, WANDA \citep{Sun2023ASA}, and SAFE \citep{lee2025safe} remove weights after pretraining while attempting to preserve the layerwise activation statistics or local curvature structure. 
Other work studies layerwise sparsity allocation and weight redistribution, such as OWL \citep{yin2023outlier}.
For vision transformers, SVITE \citep{Chen2021ChasingSI} explores end-to-end sparse training strategies.
These approaches primarily modify the sparsification procedure, allocation rule, or pruning criterion. 
HORST instead modifies the optimizer used within a standard sparse training pipeline. Our results suggest that the implicit bias of AdamW is itself a bottleneck for sparse transformers, and that composing AdamW with a entropy mirror map update can improve the sparse solutions found by existing dense-to-sparse training methods.

\section{Composable optimizer operators and their structural properties}\label{section : notation}
We treat optimization updates as functional operators, maps that take a parameter vector and a gradient (oracle) and return a step direction. This abstraction lets us reason about optimizers at the level of what their effect is on the gradients, which is precisely the level at which the dichotomy we identify becomes legible (Figure \ref{fig: dichotomy}).
The functional formulation is motivated by the modularity framework proposed in \citet{bernstein2025modular}. The key idea of modularity is that an optimizer can be applied to a group of parameters: rather
than a single operation acting on all parameters simultaneously, the operator acts independently on subsets of parameters, layers, heads, or individual weights.
In similar spirit we can compose update rules: two operators can be applied in sequence, each acting on the output of the previous one, and the result is itself a valid optimizer. 


\paragraph{Optimizer operators.}
Since our overarching goal is parameter sparsity we consider a modular objective on the parameter level $f : \mathbb{R}^n \rightarrow \mathbb{R}$ with gradient oracle $g : \mathbb{R}^n \rightarrow \mathbb{R}^n$ and parameters $\theta \in \mathbb{R}^n$. The gradient oracle and parameters together with some
geometric quantity, a norm or a metric tensor, form a triple that characterizes
the operator. Given two operators $A$ and $B$, their composition is defined by
\begin{equation*}
   AB(\theta, g) : =  (A \circ B)(\theta, g) := A\!\left(\theta,\, B(\theta, g)\right),
\end{equation*}
that is, $B$ produces a modified gradient direction and $A$ then acts on that direction instead of the gradient itself. Moreover, in general the update rule is given by:
$
    \theta_{k +1} = \theta_k - \eta A(\theta_k, g_k)
$
for $k \in [T]$, where $\eta > 0$ is the learning rate and $(\theta_k, g_k)$ are the parameter and gradient estimate at iteration $k$.

\paragraph{Operator properties.}
Before introducing the two operator classes, we define the structural properties we will use to characterize them. These properties are not desiderata we impose; rather, we will show in \S~\ref{section:properties} that each class naturally satisfies one and violates the other, and it is this asymmetry that makes composition the right tool.

\begin{definition}\label{definition : linear opperator}
    An operator $A$ is linear if and only if it satisfies the following properties:
    \begin{itemize}
        \item for any $g, h \in \mathbb{R}^n$ and $\theta \in \mathbb{R}^n$ we have $A(\theta, g+ h) = A(\theta, g) + A(\theta, h)$.
        \item for any $g \in  \mathbb{R}^n$ and $\lambda \in \mathbb{R}$ we have $A(\theta, \lambda g) = \lambda A(\theta, g)$.
    \end{itemize}
\end{definition}

\begin{definition}\label{definition : bounded operator}
    An operator $A$ is bounded if and only if there is a constant $C\geq 0$ independent of $(\theta, g)$ such that for all $\theta \in \mathbb{R}^n$ and $g \in \mathbb{R}^n$:
    \begin{equation*}
        ||A(\theta, g)|| \leq C (1+ ||\theta||)
    \end{equation*}
\end{definition}

In Definition~\ref{definition : bounded operator} we do not specify the norm; since all norms on $\mathbb{R}^n$ are equivalent. Intuitively, boundedness is a stability condition: it says that the optimizer’s step size stays controlled, instead of becoming arbitrarily large when the gradient is large.

\begin{definition}\label{definition : commuting}
    Two operators $A,B$ are commutative if their Lie Bracket $[A,B] := AB -BA = 0$.
\end{definition}

Commutativity will matter because, as we show in \S~\ref{section:properties}, the two operator classes do not commute. The order of composition therefore determines what is inherited from each class, and the specific ordering we propose is what allows HORST to simultaneously realize stability and a sparsity bias.

\paragraph{Steepest and mirror descent operators.}
We now define the two operator classes. Each is characterized by its own
geometric quantity in the triple $(g, \theta, \cdot)$: a norm $|| \cdot || : \mathbb{R}^n \rightarrow \mathbb{R}$ for steepest descent, and a mirror potential $\nabla R : \mathbb{R}^n \rightarrow \mathbb{R}^n$ for mirror descent.

\begin{definition}(Steepest operator)
    Given the triple $(g, \theta, || \cdot ||)$ the steepest descent operator $S_{|| \cdot ||} :  \mathbb{R}^n \times \mathbb{R}^n \rightarrow \mathbb{R}^n$ with respect to a norm $|| \cdot||$ is defined as:
    \begin{equation*}
        S_{|| \cdot||}(\theta, g) : = \text{arg max}_{||d|| \leq 1} \langle g(\theta),d\rangle \qquad \text{ such that } \qquad \theta_{k +1} = \theta_k - \eta  S_{|| \cdot||}(\theta_k, g_k) \qquad \text{ for } k \in [T].
    \end{equation*}
\end{definition}

\begin{example}(Sign descent)
    Consider the triple $(g, \theta, || \cdot ||_{\infty})$ then $ S_{|| \cdot||}(\theta, g) = \text{sign} \left(g(\theta) \right)$.
\end{example}

\begin{example}(Coordinate descent)
    Consider the triple $(g, \theta, \|\cdot\|_1)$; then
    $S_{\|\cdot\|}(\theta, g) = s_i\ e_i$ where $i = \operatorname*{arg\,max}_j |g_j(\theta)|$ and $e_i$ denotes the $i-$th basis vector and $s_i : = \text{sign}(g_i(\theta))$ the sign of the chosen gradient.
\end{example}

Observe that the steepest operator is scale-invariant: only the direction of $g(\theta)$ matters, not its magnitude, so doubling the gradient leaves the step unchanged. Similarly Adam also normalizes the update \citep{tsilivis2025flavors, jacobs2026never}, which is precisely what makes it stable for transformer training. As we show in \S~\ref{section:properties}, scale-invariance implies the operator being bounded but not linear.

A second way to introduce geometry into the update is via a mirror potential, a convex function $R: \mathbb{R}^n \rightarrow \mathbb{R}$ whose Hessian $\nabla^2 R(\theta)$ defines a local metric at each point. Rather than projecting the gradient onto a norm ball, the mirror operator rescales it by the inverse metric. For this to be well-defined and well-behaved, we need to make assumptions on $R$ we highlight the most important for our goal of contrasting the different implicit biases.

\begin{definition}(Inverse coercive)\label{definition : Legendre function}
Let $R : \mathbb{R}^d \rightarrow \mathbb{R} \cup \{\infty\}$ be twice differentiable  convex function.
We say that $R$ is inversely $\mu-$coercive iff there exists a constant $\mu > 0$, the coercivity constant, such that for all $\theta, \xi \in \mathbb{R}^n$:
    \begin{equation*}
        \xi^T \left(\nabla^2 R(\theta)\right)^{-1} \xi \geq \mu ||\xi||_{L_2}^2.
    \end{equation*}
\end{definition}

In Definition \ref{definition : Legendre function}  inverse coercivity is the condition that controls how aggressively the mirror map rescales the gradient, moreover if the objective $f$ satisfies the PL-inequality (Definition \ref{definition : PL inequality}) it recovers linear convergence in the gradient flow setting \citep{jacobs2024maskmirrorimplicitsparsification}. For completeness, in Definition \ref{def : legendre bregman} (Legendre and Bregman)  we recall two characterizations that make the mirror operator well-behaved.
As the examples below illustrate, gradient descent and Hyperbolic entropy mirror maps that induce an $L_2$ and $L_1$ bias are inverse coercive with $\mu =1$ and $\mu = \gamma$ respectively, while the $\cosh$-entropy that induces an $L_{\infty}-$bias is not.

\begin{definition}(Mirror operator)
    Given the triple $(g, \theta, \nabla R)$ the mirror descent operator $M_{R}: \mathbb{R}^n \times \mathbb{R}^n \rightarrow \mathbb{R}^n $ with respect to the mirror map $R$ is defined as:
    \begin{equation*}
        M_R(\theta, g) : = \nabla^2 R^{-1}(\theta) g(\theta) \qquad \text{such that} \qquad \theta_{k + 1} = \theta_k - \eta M_R(\theta_k, g_k) \qquad \text{for } k \in [T].
    \end{equation*}
\end{definition}

\begin{example}(Gradient descent, $L_2$-bias)
    Consider $R(\theta) = \frac{1}{2}|| \theta||_2^2$ then $M_R(\theta,g) = g(\theta)$ and satisfies all properties in Definition \ref{definition : Legendre function}.
\end{example}

\begin{example}(Hyperbolic entropy, $L_1$-bias)
    Consider $R_{\gamma}(\theta) = \sum_{i =1}^n \theta_i \text{arcsinh}(\frac{\theta_i}{\gamma}) - \sqrt{\theta_i^2 + \gamma^2}$ then $M_R(\theta, g)  = \sqrt{\theta^2 + \gamma^2} g(\theta)$ and satisfies all properties in Definition \ref{definition : Legendre function} when $\gamma > 0$. Moreover this is related to the entropy mirror map $\nabla R(\theta) = \log \theta$ by taking $\gamma \rightarrow 0$.
\end{example}

\begin{example}($\cosh$-entropy, $L_{\infty}$-bias \citep{pesme2024implicit})
    Consider $R(\theta) = \sum_{i =1}^n \cosh(\theta_i)$,  then $M_R(\theta, g)  = \frac{1}{\cosh(\theta)} g(\theta)$. This is a Legendre function and is
    admissible as a mirror map, but fails inverse $\mu$-coercivity: the rescaling $1/\cosh(\theta)$ vanishes as $|\theta| \rightarrow \infty$, which progressively suppresses gradient updates for large weights and can lead to slow convergence.
\end{example}

\paragraph{Implicit bias.}

We now recall what bias each operator class induces. We work in the setting of binary classification on separable data with $K$ points $\{x_i, y_i\}_{i=1}^K$,
$(x_i, y_i) \in \mathbb{R}^n \times \{\pm 1\}$, exponential loss $\ell_i(\theta)
:= \exp(-y_i \langle \theta, x_i \rangle)$, and total objective $\mathcal{L}(\theta)
:= \sum_{i \in [K]} \ell_i(\theta)$. 
Moreover, we let the $\eta \rightarrow 0$ and consider the flow setting. On linearly separable data, the resulting iterates $\theta_t$ diverge along a fixed limiting direction $\bar{\theta} := \lim_t \theta_t / \|\theta_t\|$ \citep{gunasekar2020characterizing}, which the implicit bias characterizes. In general $\bar{\theta}$ converges in direction to the a KKT point of an optimization problem with objective $C(\theta): \mathbb{R}^n \rightarrow \mathbb{R}$:
\begin{equation*}
    \min\, C(\theta) \qquad \text{such that} \qquad y_i \langle \theta, x_i
    \rangle \geq 1 \quad \text{for all } i \in [K].
\end{equation*}
For steepest descent $C(\theta) = \|\theta\|$ \citep{tsilivis2025flavors} and for mirror descent $C(\theta) = \phi_\infty(\theta)$  \citep{pesme2024implicit} which is the so-called horizon function associated to $R$. In the steepest descent case the implicit bias correspondence with sign and coordinate descent is immediately clear.
In the mirror descent case, for our goal, we only need to know that $\phi_\infty(\theta) \sim || \theta||_1$ for the hyperbolic entropy and $\phi_\infty(\theta) \sim || \theta||_{\infty}$ for the $\cosh$-entropy.

\section{Operator characterization}\label{section:properties}
We first characterize the two operator classes defined in \S~\ref{section : notation}. The steepest operator is well-behaved in terms of step size, it always returns a bounded direction, but it achieves this by throwing away gradient magnitude entirely, which is a nonlinear operation. The mirror operator, by contrast, rescales the gradient smoothly and proportionally,
which is linear, but has no built-in ceiling on how large the step can be.

\begin{lemma}[See also Lemma \ref{lemma : steepest descent}.]\label{lemma:steepest-bounded-nonlinear}
    The steepest operator is bounded but not linear.
\end{lemma}

\begin{lemma}[See also Lemma \ref{lemma : mirror linear opperator}]\label{lemma:mirror-linear-unbounded}
    The mirror operator is linear and unbounded in general.
\end{lemma}

Together, Lemmas~\ref{lemma:steepest-bounded-nonlinear}
and~\ref{lemma:mirror-linear-unbounded} confirm the dichotomy summarized in
Figure~\ref{fig: dichotomy}: each class has exactly the property the other lacks.
Explaining why the $\cosh$-entropy fails as a route to $L^\infty$ stability: to equalize gradient coordinates via a linear rescaling, the inverse Hessian of $R$ must damp large components and amplify small ones, which forces it to decay for large weights and lose coercivity, and with that linear convergence \citep{jacobs2024maskmirrorimplicitsparsification}. The steepest operator achieves the same bias through scale-invariance rather than linear rescaling, and is bounded by construction. Stability is therefore better achieved by steepest descent, not mirror descent.

\paragraph{Non-commutative and inheritance.}
Since the two classes have complementary properties, a natural question is whether composing them yields an operator that inherits both. The answer depends critically on the order of composition: one ordering leads to steepest operator alone, while the other, used in HORST, leverages both geometries. This is due to the non-commutative nature between the operator classes.

\begin{lemma}\label{lemma : sign mirror redundant mirror}
    Consider $S_{|| \cdot ||_{\infty}}$ and $M_R$ with $R$ any separable mirror map then $S_{|| \cdot ||_{\infty}} M_R = S_{|| \cdot ||_{\infty}}$.
\end{lemma}
Proof.
This follows from a direct calculation. We know that $R$ is separable and its Hessian is positive definite and $S_{|| \cdot||_{\infty}}$ corresponds to sign descent this gives for any $\theta \in \mathbb{R}^n$ and $g \in \mathbb{R}^n$:
\begin{align*}
    S_{|| \cdot ||_{\infty}} M_R (\theta, g) = \text{sign}(\nabla^2 R^{-1}(\theta) g(\theta)) = \text{sign}(g(\theta)) = S_{|| \cdot ||_{\infty}}(\theta, g). \qquad \square
\end{align*}


\begin{corollary}\label{lemma : commuting operators}
    The mirror and steepest operator do not commute i.e. $[S_{|| \cdot ||}, M_R] \neq 0$.
\end{corollary}
Proof. 
Consider the the setting of Lemma \ref{lemma : sign mirror redundant mirror}. Then we know that $S_{|| \cdot ||_{\infty}} M_R = S_{|| \cdot ||_{\infty}}$. But now consider a mirror map that is not $\frac{1}{2}||\theta ||_2^2$ we have that there is a $\theta \in \mathbb{R}^n$ such that:
\begin{equation*}
   M_R S_{|| \cdot ||_{\infty}}(\theta ,g)  = \nabla^2 R^{-1}(\theta)\ \text{sign}(g(\theta)) \ne \text{sign}(g(\theta)).
\end{equation*}
Therefore we do not have $[S_{|| \cdot ||_{\infty}}, M_R] = 0$ leading to a contradiction. $\square$

Lemma~\ref{lemma : sign mirror redundant mirror} and Corollary~\ref{lemma : commuting operators}
show that the order of composition matters: if we apply the mirror map first, sign descent removes the effect of the mirror map. If we
apply it second, the mirror map gets to act on the sign step and can potentially introduce the $L_1$ sparsity bias. 

Thus, we have shown that the steepest operator is bounded but not linear, and the mirror operator is linear but not bounded. Each class has exactly what the other lacks, and composing them in the right order is what allows HORST to get the best of both worlds as we will show next in Corollary \ref{corollary:composition-bounded}.

\begin{corollary}\label{corollary:composition-bounded}
    For any norm $\|\cdot\|$ and any separable mirror map $R$ such that $|| \nabla^2 R^{-1}(\theta) || \leq C(1 + ||\theta||)$ for all $\theta \in \text{dom} R$, the composed operator $M_R S_{\|\cdot\|}$ is bounded.
\end{corollary}
Proof. This follows from boundedness of the steepest descent operator and linearity of the mirror operator i.e. $\|M_R S_{\|\cdot\|}(\theta, g)\| = \|\nabla^2 R^{-1}(\theta)\,
    S_{\|\cdot\|}(\theta, g)\| \leq C_1 \|\nabla^2 R^{-1}(\theta)\| \
    \|S_{\|\cdot\|}(\theta, g)\| \leq C_1\|\nabla^2 R^{-1}(\theta)\| \leq C_2(1 + ||\theta||)$.
$\square$

\begin{remark}
    The growth condition in Corollary \ref{corollary:composition-bounded} on the inverse hessian $\nabla^2 R^{-1}$ is satisfied by all considered example mirror maps: gradient descent, hyperbolic entropy and $\cosh$ entropy.
\end{remark}


\paragraph{Implicit bias of composition.}
Motivated by Corollary \ref{corollary:composition-bounded} we propose to use the bounded composed operator $M_R S_{\|\cdot\|}$:
\begin{equation*}
    M_R S_{\|\cdot\|}(\theta, g) = \nabla^2 R^{-1}(\theta)\,
    \operatorname*{arg\,max}_{\|d\| \leq 1} \langle g(\theta), d \rangle  \qquad \theta_{k + 1} = \theta_k - \eta M_R S_{|| \cdot ||}(\theta_k, g_k) \qquad \text{for } k \in [T].
\end{equation*}
Concretely, we focus here on the entropy mirror map $\nabla R(\theta) = \text{log} \ \theta$ and steepest descent norm $|| \cdot ||_p$ with $p \in [2, \infty)$. These are closely related to the Hyperbolic entropy $(\gamma \rightarrow 0)$ and sign descent $(p \rightarrow \infty)$. We can characterize the implicit bias in the binary classification setting on linear separable data for the continuous-time flow ($\eta \rightarrow 0$), described by the steepest-mirror differential inclusion:
\begin{equation}\label{eq:mirror-steepest}
    \frac{d\theta_t}{dt} \in -\nabla^2 R^{-1}(\theta_t)\,
    \text{sign}(\nabla \mathcal{L}(\theta_t))|\nabla\mathcal{L}(\theta_t)|^{q-1} /||\nabla\mathcal{L}(\theta_t)||_q \ \ \text{ where } q \text{ is such that } \frac{1}{p} + \frac{1}{q} =1.
\end{equation}
\begin{theorem}\label{theorem:implicit-bias}
    Consider steepest-mirror descent with respect to the $L_p$-norm, $p \in
    \mathbb{N}_{\geq 2}$, and mirror map $\nabla R(\theta) = \log(\theta)$. Then the
    iterates of Eq.~\eqref{eq:mirror-steepest} converge in direction to a KKT
    point of:
    \begin{equation*}
        \min\, \|\theta\|_1 \qquad \text{such that} \qquad y_i \langle \theta,
        x_i \rangle \geq 1 \quad \text{for all } i \in [K].
    \end{equation*}
\end{theorem}

\begin{proof}
   This can be shown by using the steepest mirror flow connection to homogeneous reparameterization as developed in \citep{jacobs2026never} and combining it with the steepest flow characterization of the max margin in \citep{tsilivis2025flavors}. For the full proof see Theorem \ref{theorem : appendix implicit bias}.
\end{proof}

Theorem  \ref{theorem:implicit-bias} establishes that the entropy mirror map $\nabla R(\theta) = \text{log} \ \theta$ can overwrite the implicit bias of steepest descent with respect to an $L_p$-norm for $p \in [2, \infty)$. Moreover, this suggests that the sparsity bias may extend to the limiting regime ($p \rightarrow \infty$, $\gamma \rightarrow 0$). However, establishing this limit rigorously remains an open problem.

\begin{algorithm}[ht]
\caption{HAM}\label{alg: ham}
\begin{algorithmic}
\Require steps $T$, schedule $\eta$, initialization $\theta_\mathrm{init}$, extra regularization $\beta \geq 0$, weight decay $\lambda \geq 0$, relative lr $\alpha \geq 0$.
\For{$k \in 0 \ldots T-1$}
    \State $\theta_{k+\frac{1}{2}} = \theta_k - \mathrm{Adam}\!\left(g(\theta_k),\, \eta\right) - \eta \lambda \theta_k$
    \State $\theta_{k+1} = \theta_{k + \frac{1}{2}}\exp\!\left(
    -\eta \ \left( \alpha \ \text{sign}(\theta_{k + \frac{1}{2}})\ g(\theta_k)+ \beta\right)\right) \qquad \qquad \qquad \ \ \ $  
\EndFor \\
\Return Model weights $\theta_T$
\end{algorithmic}
\end{algorithm}

\section{HORST: optimizer for sparse transformers}\label{sec:algorithm}
We first recall the HAM optimizer scheme  \citep{jacobs2025hamhyperbolicstepregulate} in Algorithm \ref{alg: ham}. This alternates a base optimizer (AdamW) with an exponential step which corresponds to using the mirror map $\nabla R(\theta) = \log(\theta)$ as used in Theorem~\ref{theorem:implicit-bias}. 
Note that this entropy mirror map does not have full domain $\mathbb{R}^n$ and is not inversely coercive on its own motivating the alternating scheme. 
Next, the hyperparameter $\alpha \ge 0$ is used to determine the relative importance of the updates.
Moreover, extra (decoupled) regularization is used to induce sparsity as a default captured by the hyperparameter $\beta \geq 0$.
This motivates three design choices for HORST to turn the composition into a practical optimizer:
\begin{itemize}
    \item We replace the steepest descent step with Adam in the exponential update.
    \item We use decoupled weight decay with strength $\lambda > 0$.
    \item We adopt an alternating scheme from \citep{jacobs2025hamhyperbolicstepregulate} to recover inverse coercivity that applies the exponential update rule after the standard optimizer step, with a hyperparameter $\alpha \geq 0$ that controls the relative strength of the sparsity bias.
\end{itemize}

Note that setting $\alpha = 0$ recovers the base optimizer (AdamW), while increasing $\alpha$ strengthens the sparsity bias.
This gives the update scheme in Algorithm~\ref{alg:cap}. Similar as HAM, HORST requires no additional gradient evaluations; its extra cost is a per-parameter exponential update.
Moreover, $\eta$ is the learning rate, $\lambda$ the decoupled weight decay strength, and for completeness $\beta \geq 0$ the extra regularization parameter as in HAM. Observe that $\beta$ for HORST is doing the same work as the decoupled weight decay $\lambda > 0$.

\begin{algorithm}[ht]
\caption{HORST}\label{alg:cap}
\begin{algorithmic}
\Require steps $T$, schedule $\eta$, initialization $\theta_\mathrm{init}$,
extra regularization $\beta \geq 0$, weight decay $\lambda \geq 0$, relative lr $\alpha \geq 0$.
\For{$k \in 0 \ldots T-1$}
    \State $\theta_{k+\frac{1}{2}} = \theta_k - \mathrm{Adam}\!\left(g(\theta_k),\, \eta\right) - \lambda\eta \theta_k$
    \State $\theta_{k+1} = \theta_{k + \frac{1}{2}}\exp\!\left(
    -\alpha\ \text{sign}(\theta_{k+ \frac{1}{2}}) \ \mathrm{Adam}\!\left(g(\theta_{k}),\, \ \eta\right) -\eta \beta\right) \qquad \qquad \ \ \ $  (Composition)
\EndFor  \\
\Return Model weights $\theta_T$
\end{algorithmic}
\end{algorithm}

\section{Experimental validation}\label{section : experiments}

\paragraph{Toy illustration.}
We first show the importance of composing on Adam with the entropy mirror map itself. We consider binary linear classification on linear separable data. The target solution is sparse and consists of two features. We describe the details of the setting in Appendix \ref{appendix : linear class}. As we show in Figure \ref{fig:acdc weight dist} (Right), using an entropy mirror map after Adam (Exp-Adam) leads to more sparsity, recovering the target solution better highlighted by the green band. In contrast composing in the opposite way (Adam-Exp) leads to overfitting as was the case for Adam in Figure \ref{fig:placeholder}.

\begin{figure}[ht]
    \centering
    \includegraphics[width=0.35\linewidth]{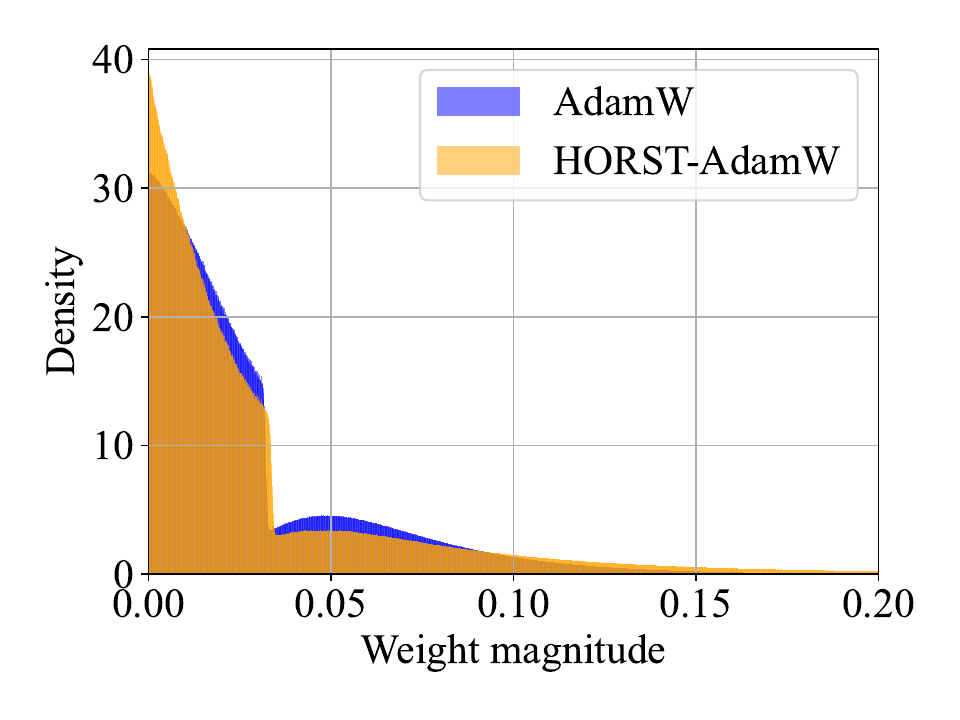}
    \qquad
    \includegraphics[width=0.4\textwidth]{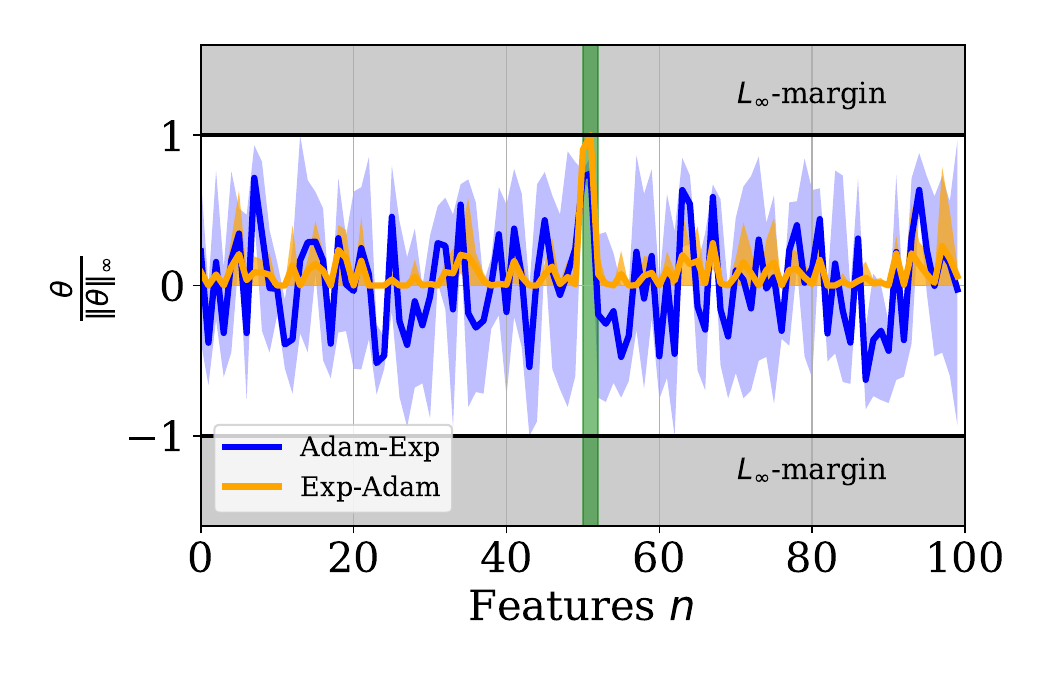}
    \caption{(Left) The final unmasked weight distribution of a DeiT-base trained with AC/DC to $70\%$ sparsity. HORST successfully concentrates more weights at zero finding better generalizable sparse solutions. (Right) Linear separable binary classification with a diagonal linear network $\langle \theta, x \rangle$, Exp-Adam learns a sparser representation while Adam-Exp does not, this is in line with Lemma \ref{lemma : sign mirror redundant mirror}.}
    \label{fig:acdc weight dist}
\end{figure}

\paragraph{Vision and language transformer experiments.}
We evaluate HORST-AdamW on both vision and language transformer settings to test whether its sparsity-inducing optimizer geometry improves performance under pruning. Unless otherwise noted we use the hyperparameters $\alpha =5, \beta =0$ for HORST.\footnote{The * indicates that $\beta =1e-2$ was used instead of $\beta =0$, which gave a slightly better improvement.}
In the vision experiments, HORST-AdamW is combined with the AC/DC sparsification \citep{peste2021acdc} pipeline and applied to DeiT-base and DeiT-small models on ImageNet \citep{imagenet} across high sparsity levels, for details see Appendix \ref{appendix : details vision}. Compared with the AdamW baseline, HORST-AdamW consistently achieves higher validation accuracy, with the largest gains appearing at $80-90\%$ sparsity reported in Table \ref{tab:deit_sparsity_50_70}. 
In the language experiments, dense GPT-2 Small \citep{radford2019language} checkpoints trained on SlimPajama-6B \citep{cerebras2023slimpajama} are evaluated using one-shot layerwise unstructured magnitude pruning without fine-tuning. HORST-AdamW yields lower validation perplexity after pruning than AdamW (Figure \ref{fig:gpt2_small_scale_perplexity}), indicating that the proposed optimizer produces weights that are more robust to sparse compression, for further details see Appendix \ref{appendix : details language}. Overall, these experiments show that composing AdamW with a hyperbolic mirror-style update improves sparse transformer performance across both vision and language tasks while noticeably changing the global weight distribution as shown in Figure \ref{fig:acdc weight dist} and \ref{fig:weight_distribution_small_scale_gpt2_on_slimpajama_dense_7layer_vs_12layer}.

Moreover, we include an additional experiment comparing with HAM in Table~\ref{table : horst ham comparison}. On DeiT-small at $90\%$ sparsity, HAM provides a modest improvement over the AdamW baseline, whereas HORST substantially outperforms both HAM and the baseline. This supports the importance of composing the exponential mirror-style update with the AdamW update itself, rather than applying the gradient-based exponential step used in HAM.

\begin{table}[h]
\caption{Performance in terms of validation accuracy for a DeiT model with different optimizers and the AC/DC sparsification method \citep{peste2021acdc}.}
\label{tab:deit_sparsity_50_70}
\centering
\begin{tabular}{llccc}
\hline
Model & Optimizer & $70\%$ & $80\%$ & $90\%$ \\
\hline
DeiT-base  & AdamW        & $76.15 \pm 0.71$ & $74.18 \pm 0.29$ & $68.86 \pm 2.05$ \\
DeiT-base  & HORST-AdamW  & $\mathbf{78.89 \pm 0.79}$   & $\mathbf{78.63 \pm 1.26}$ & $\mathbf{76.83 \pm 0.92}$\\
\hline
DeiT-small & AdamW        & $71.59 \pm 1.12$ & $67.07 \pm 0.31$ & $59.64 \pm 1.52$
\\
DeiT-small & HORST-AdamW  & $\mathbf{78.42 \pm 0.39}$ &  $\mathbf{75.83 \pm 0.47}$* & $\mathbf{69.41 \pm 2.75}$*\\
\hline
\end{tabular}
\end{table}


\section{Conclusion and future work}
\begin{wrapfigure}{r}{0.45\linewidth}
    \centering
    \includegraphics[width=0.9\linewidth]{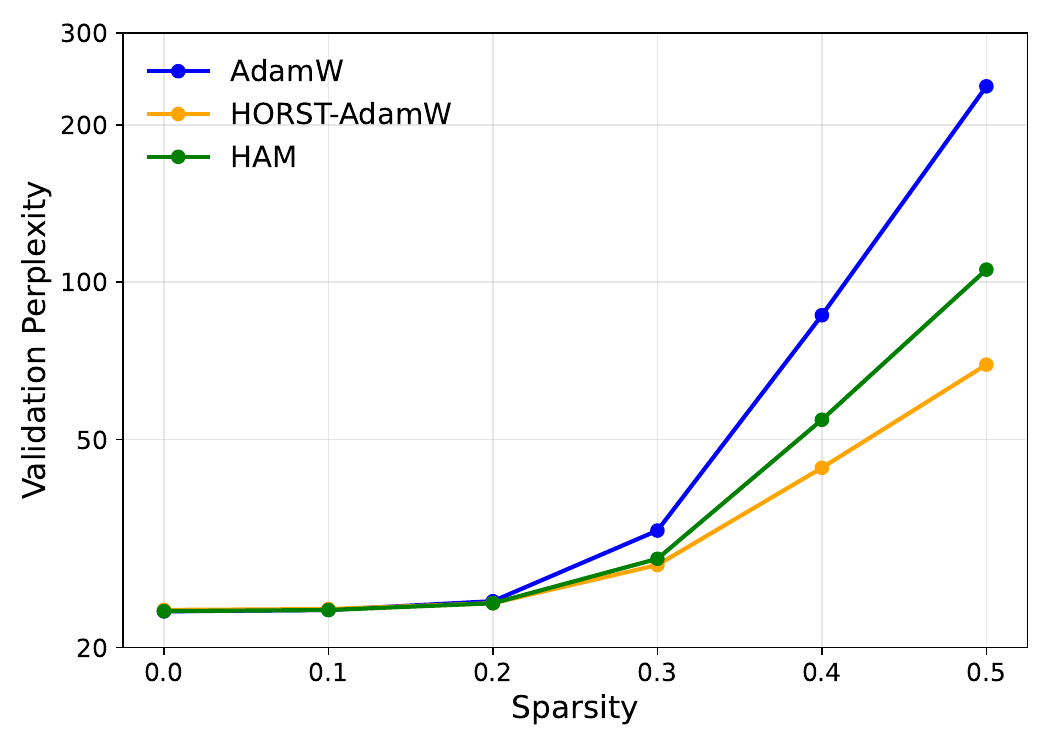}
    \caption{One-shot layerwise unstructured magnitude pruning of dense GPT-2 Small ($\approx$ 124M params) checkpoints trained on SlimPajama-6B with AdamW vs. HAM vs. HORST-AdamW; no fine-tuning. HORST-AdamW consistently achieves lower validation perplexity than both. See Table \ref{tab:gpt2_small_one_shot_pruning_experiment_val_perplexity_table} for details.}
    \label{fig:gpt2_small_scale_perplexity}
\end{wrapfigure}

We have developed a compositional perspective on optimization by treating update rules as operators and analyzing their geometric and algebraic structure. This reveals a fundamental dichotomy: steepest descent operators are bounded but nonlinear and naturally associated with stability and $L_\infty$-type implicit bias, while mirror descent operators are linear but generally unbounded and can induce sparsity-promoting $L_1$ structure through the (hyperbolic) entropy mirror map.

A key insight is that these two operator classes are non-commutative, making the order of composition important. When composed correctly, the resulting dynamics can inherit both the stability of steepest descent and the sparsity bias of mirror descent, leading to an effective mechanism for controlling implicit regularization through operator design.
Beyond the theoretical characterization, the primary practical outcome of this framework is a consistent improvement in sparse training performance. Across transformer settings trained with sparsification pipelines, the composed optimizer significantly improves accuracy for all sparsity levels, with the largest gains in the high sparsity regime where standard adaptive optimizers degrade sharply. This shows that the implicit bias of the optimizer plays an important role in how well the model can be pruned.
However, several directions remain open. First, the $L_\infty$ limit of steepest descent and its interaction with mirror-based updates is not yet fully understood formally, closing this gap would provide a complete characterization of the limiting dynamics underlying sign-based adaptive methods. Second, a promising direction is understanding how spectral descent interacts with steepest-mirror compositions, and whether similar implicit bias control can be achieved. Finally, combining HORST with post-training pruning methods such as SparseGPT is a promising direction, since HORST modifies the geometry of the learned checkpoint while such methods improve the pruning rule applied after training.

\newpage

\paragraph{Acknowledgments}
We gratefully acknowledge the Gauss Centre for Supercomputing e.V. for
funding this project by providing computing time on the
GCS Supercomputer JUWELS at Jülich Supercomputing
Centre (JSC) \citet{juelich2021juwels}. We are also grateful for funding from
the European Research Council (ERC) under the Horizon
Europe Framework Programme (HORIZON) for proposal
number 101116395 SPARSE-ML.

\bibliographystyle{unsrtnat}
\bibliography{neurips_2026}

\newpage
\appendix

\section{Extended related work}
A complementary line of sparse-training works uses the mirror-descent/reparameterization mechanism described above to induce sparsity through optimization dynamics rather than through pruning alone. 
First, \citet{Ziyin2022spredSL} showed that $L_1$-type sparse solutions can be obtained by running vanilla SGD on a Hadamard factorization, while \citet{kolb2025deep} study how increasing the depth of such pointwise factorizations can drive extreme sparsity, sometimes at the cost of performance. 
Moreover, \citet{jacobs2024maskmirrorimplicitsparsification} show that continuous sparsification with the reparameterization $m \odot w$ undergoes an implicit-bias transition from $L_2$ to $L_1$.

\section{Preliminaries implicit bias}

\paragraph{Convex optimization}
We first recall the Bregman divergence, which naturally generalizes the Euclidean distance metric.

\begin{definition}[Bregman divergence]
    Given a differentiable convex function $R : \mathbb{R}^n \rightarrow \mathbb{R}$,
    the Bregman divergence associated to $R$ is:
    $
        D_R(\theta, \xi) := R(\theta) - R(\xi) - \langle \nabla R(\xi),
        \theta - \xi \rangle.
    $
\end{definition}

Moreover we give two basic desired characterizations of the mirror maps in Definition \ref{def : legendre bregman}.

\begin{definition}(Legendre and Bregman function)\label{def : legendre bregman}
     We say $R$ is a Legendre function when the following holds:
 \begin{itemize}
     \item  $R$ is strictly convex on the interior of its domain $\text{int} (\text{dom} R)$.
    \item For any sequence $\{\theta_i\}^{\infty}_{i = 1}$ going to the boundary of $\text{dom} R$, the gradient diverges, i.e.  $\lim_{i\rightarrow \infty}||\nabla R(\theta_i)|| = \infty$.
 \end{itemize}
We say $R$ is a Bregman function if 
\begin{itemize}
    \item For any $\theta\in \text{dom} R$ and $\gamma \in \mathbb{R}$, $\{\xi \in \text{dom} R | D_R(\theta,\xi) \leq \gamma\}$ is bounded.
     \item For any $\theta \in \text{dom} R$ and sequence $\{\theta_i\}^{\infty}_{i=1} \subset \text{int}(\text{dom} R)$ such that $\lim_{i \rightarrow \infty} \theta_i = \theta$, it holds that $\lim_{i\rightarrow \infty} D_R(\theta,\theta_i) \rightarrow 0$.
\end{itemize}
\end{definition}

For standard convergence proofs \citep{wojtowytsch2021stochastic} the PL-inequality can be used.
\begin{definition}(Polyak-\L ojasiewicz inequality)\label{definition : PL inequality}
    Let $\Lambda > 0$ be a constant and $\mathcal{L} \in C^1(\mathbb{R}^n, \mathbb{R})$ then we say $\mathcal{L}$ satisfies the PL-inequality if and only if:
    \begin{equation*}
        ||\nabla \mathcal{L}(\theta)||^2_2 \geq 2 \Lambda (\mathcal{L}(\theta)- \mathcal{L}(\theta^*)) \text{ for all } \theta \in \mathbb{R}^n,
    \end{equation*}
    where $\theta^*$ is a minimizer of $\mathcal{L}$, which is assumed to exist for $\mathcal{L}$.
\end{definition}

Assuming Definition \ref{definition : PL inequality} linear convergence for mirror flow:
\begin{equation*}
    d \theta_t = -\nabla^2 R(\theta_t)^{-1} \nabla \mathcal{L}(\theta_t) dt , \qquad \theta_0  = \theta_{\text{init}},
\end{equation*}
under the inversely coercive mirror map Assumption \ref{def : legendre bregman} can be shown:
\begin{equation*}
    d \mathcal{L} (\theta_t) = - \nabla \mathcal{L}(\theta_t) \nabla^2 R(\theta_t)^{-1} \nabla \mathcal{L}(\theta_t) dt \leq  \nabla - \mu || \mathcal{L}(\theta_t) ||_2^2 dt \leq \mu 2\mathcal{L}\Lambda (\theta_t) dt
\end{equation*}
where we assumed the minimum is attained at zero for simplicity.
Then the result directly follows from applying Gronwall's Lemma as used in \citep{jacobs2024maskmirrorimplicitsparsification}.

\paragraph{Implicit bias}
Now, we provide more details on the implicit of both operator classes. Recall that the setting is binary classification on linear separable data with $K$ points $\{x_i, y_i\}_{i=1}^K$,
$(x_i, y_i) \in \mathbb{R}^n \times \{\pm 1\}$, exponential loss $\ell_i(\theta)
:= \exp(-y_i \langle \theta, x_i \rangle)$, and total objective $\mathcal{L}(\theta)
:= \sum_{i \in [K]} \ell_i(\theta)$. 
We note that \citep{tsilivis2025flavors} result is more general and covers also homogeneous neural networks, which we actually use in the proof as we rely on a homogeneous reparameterization.
Moreover, we let $\eta \rightarrow 0$ and consider the flow setting.
On linearly separable data, the resulting iterates $\theta_t$ diverge along a fixed limiting direction $\bar{\theta} := \lim_t \theta_t / \|\theta_t\|$ \citep{gunasekar2020characterizing}.

For steepest descent with respect to $\|\cdot\|$, this direction solves the max-margin problem \citep{tsilivis2025flavors}:
\begin{equation*}
    \min\, \|\theta\| \qquad \text{such that} \qquad y_i \langle \theta, x_i
    \rangle \geq 1 \quad \text{for all } i \in [K].
\end{equation*}
For mirror descent with a separable coercive map $R$, it solves the analogous problem in the geometry of $R$ \citep{sun2022mirror,
pesme2024implicit}:
\begin{equation*}
    \min\, \phi_\infty(\theta) \qquad \text{such that} \qquad y_i \langle \theta,
    x_i \rangle \geq 1 \quad \text{for all } i \in [K],
\end{equation*}
where $\phi_\infty(\theta) : = \lambda \lim_{\eta \rightarrow 0} \eta \  r^{-1}(R(\theta/\eta)) $, where $\lambda > 0$ is a fixed constant and $r$ is the mirror potential of a single coordinate. This is the so-called horizon function associated to $R$. 
Sign descent and the hyperbolic entropy therefore induce $L^\infty$ and $L^1$ biases respectively.

\section{Main results}
Here we provide full proofs of the statements in the main text.

\paragraph{Boundedness and linearity.}
First we show the basic properties of the operators in Lemmas \ref{lemma : mirror linear opperator} and \ref{lemma : steepest descent}.

\begin{lemma}\label{lemma : steepest descent}
    The steepest operator is bounded but not linear.
\end{lemma}
Proof.
\textbf{Linearity:}
A concrete counter example for linearity is the steepest descent operator with respect to $|| \cdot||_{\infty}$ giving for any $g \in \mathbb{R}^n$ and positive $\lambda > 0$:
\begin{equation*}
    S_{|| \cdot ||_{\infty}}(\theta, \lambda g(\theta)) = \text{sign} (\lambda g(\theta)) = \text{sign}(g(\theta)) \neq \lambda \text{sign}(g(\theta)),
\end{equation*}
which is unequal except when $\lambda =1$.
In general, by definition the steepest descent operator is scale invariant violating linearity.

\textbf{Bounded:}
By definition the steepest descent operator is maximum argument on bounded norm ball therefore it is bounded. $\square$

\begin{lemma}\label{lemma : mirror linear opperator}
    The mirror operator is linear and unbounded in general. 
\end{lemma}
Proof.
\textbf{Linearity:}
Using the definition consider two gradient estimates $g, h \in \mathbb{R}^n$ evaluated at $\theta \in \mathbb{R}^n$ and a coefficient $\lambda \in \mathbb{R}$ then first we have:
\begin{equation*}
    M_R(\theta, g+ h) = \nabla^2 R^{-1}(\theta) \left(g(\theta) + h(\theta)\right) = \nabla^2 R^{-1}(\theta) g(\theta) + \nabla^2 R^{-1}(\theta) h(\theta) = M_R(\theta, g) + M_R(\theta, h).
\end{equation*}
Moreover, we have
\begin{equation*}
    M_R(\theta, \lambda g) = \nabla^2 R^{-1}(\theta) \left(\lambda g(\theta)\right) = \lambda \nabla^2 R^{-1}(\theta) \left(g(\theta)\right)  = \lambda M_R(\theta, g). \qquad\square
\end{equation*}
\textbf{Unbounded:}
Consider gradient descent. We can choose a function $f(\theta) = \text{log}(\theta)$ which has gradient $f'(\theta) = 1/\theta$ for which we have $ \lim_{\theta \rightarrow 0^+} f'(\theta) = \infty$ and thus unbounded. $\square$

\paragraph{Implicit bias.}
We can now characterize the implicit bias of the composed operator.
\begin{theorem}\label{theorem : appendix implicit bias}
    Consider steepest-mirror descent with respect to the $L_p$-norm, $p \in
    \mathbb{N}_{\geq 2}$, and mirror map $\nabla R(\theta) = \log(\theta)$. Then the
    iterates of Eq.~\eqref{eq:mirror-steepest} converge in direction to a KKT
    point of:
    \begin{equation*}
        \min\, \|\theta\|_1 \qquad \text{such that} \qquad y_i \langle \theta,
        x_i \rangle \geq 1 \quad \text{for all } i \in [K].
    \end{equation*}
\end{theorem}
Proof.
This can be shown by using the steepest mirror flow connection to homogeneous reparameterization as developed in \citep{jacobs2026never} and combining it with the steepest flow characterization of the max-margin in \citep{tsilivis2025flavors}.
Using a time reparameterization $dt = ||\nabla \mathcal{L}(\theta_s) ||_{q} ds$ such that $\frac{1}{p} + \frac{1}{q} = 1$ the flow can be written as:
\begin{equation}\label{equation : steepest mirror log LP}
     \frac{d \theta_t}{dt} \in -|\theta| \text{sign}\left(\partial\mathcal{L}(\theta_t)\right)  | \partial\mathcal{L}(\theta_t) |^{q-1} .
\end{equation}
We now show equivalence of this flow to a re-parameterized standard steepest descent flow.
Consider the reparameterization $\theta = \Pi_{i \in [p]} w_i$ and we train with steepest descent with respect to $L_p$-norm (with time rescaled as above). 
Moreover, we initialize such that $|w_{i,0}| = |\theta_0|^{1/p}$ i.e. balanced.
Then we have the following flow equation for each $i \in [p]$:
\begin{equation*}
     \frac{d w_{i,t}}{dt} \in - \text{sign}\left(\partial_{w_i}\mathcal{L}(\theta_t)\right)  | \partial_{w_i}\mathcal{L}(\theta_t) |^{q-1} .
\end{equation*}
Using the balance equation as in \citep{jacobs2026never} ($\forall t\geq 0$, we have almost everywhere $|w_{i,t}| =  |\theta_t|^{1/p}$ ). The evolution of $\theta_t = \Pi_{i \in [p]} w_i$ is equal to Eq.(\ref{equation : steepest mirror log LP}). It follows from Theorem 3.4 in \citep{tsilivis2025flavors} that the iterates $(w_1, \hdots, w_p)$ converge in direction to the following optimization problem:
 \begin{equation*}
          \frac{1}{2}||(w_1, \hdots, w_p )||_{L_p}^2 \qquad \text{such that } \qquad y_i \langle \theta , x_i \rangle \geq 1 \text{ for all } i \in [K].
\end{equation*}
As the map $z \rightarrow z^{2/p}$ for $p > 0$ is strictly increasing for $z \geq 0$ the optimization is equivalent to:
\begin{equation*}
         \frac{1}{2} \sum_{i \in [p]}||w_i||_{L_p}^p\qquad \text{such that } \qquad y_i \langle \theta , x_i \rangle \geq 1 \text{ for all } i \in [K].
\end{equation*}
Combining this with the balance equation (which also holds for the normalized iterates) the optimization problem above is equivalent to:
 \begin{equation*}
          \frac{p}{2} || \theta||_{L_1} \qquad \text{such that } \qquad y_i \langle \theta , x_i \rangle \geq 1 \text{ for all } i \in [K],
\end{equation*}
which concludes the result as positive multiplicative constant does not alter the solution of the objective. $\square$

\newpage
\section{Implicit bias of operator classes on sparse linear classification}\label{appendix : linear class}
As specified in the main text we consider binary classification with linear separable data. More specifically, the data can be separated by a sparse support vector. Let $\mathcal{S} = \{\mathbf{x_i}, y_i\}_{i=1}^{N}$ be a dataset of i.i.d. samples with $\mathbf{x_i} \in \mathbb{R}^D$ drawn from $\mathcal{N}(0, I_D)$ and labels $y_i \in \{\pm 1\}$ for all $i \in [N]$ generated by a sparse teacher $\theta_\star \in \mathbb{R}^{D}$ via $y_i = \operatorname{sign}\langle \theta_\star, x_i\rangle$, where $\theta_\star$ is supported only on its first two coordinates. The learner is the linear model $f(x;\theta) = \langle \theta, x\rangle$ with $\theta \in \mathbb{R}^D$,  trained under the empirical exponential margin loss $\mathcal{L}(\theta) = \frac{1}{N}\sum_{i=1}^{N} e^{-y_i \langle \theta, x_i\rangle}$. We fix $D = 100, N = 80$, placing the problem in the over-parametrized regime where infinitely many $\theta$ achieve zero training loss and the implicit bias of the optimizer determines which one is selected. 


We use this setting to illustrate the dichotomy: each optimizer in our framework is ran for a fixed budget of full-batch steps and we compare the directions $\bar{\theta}$ they select to the sparse teacher $\theta_\star$. Optimizer with an $L_1$ sparsity bias should recover $\theta_\star$. In contrast, an optimizer with an $L_\infty$ bias will spread mass across all $D$ coordinates, saturating at the $\pm ||\theta||_\infty$ margin.

\begin{figure}[h]
     \centering
     \begin{subfigure}[b]{0.49\linewidth}
         \centering
         \includegraphics[width=\linewidth]{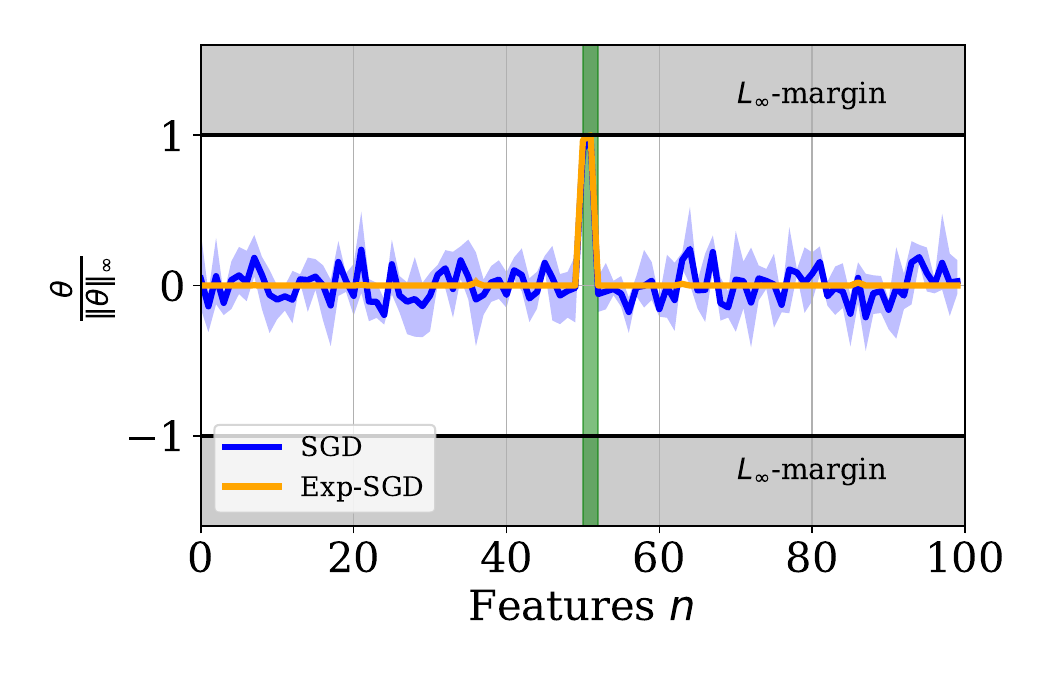}
         \caption{SGD vs Exp-SGD.}
         \label{fig:sgd_vs_exp_sgd_implicit_bias}
     \end{subfigure}
     \hfill
     \begin{subfigure}[b]{0.49\linewidth}
         \centering
         \includegraphics[width=\linewidth]{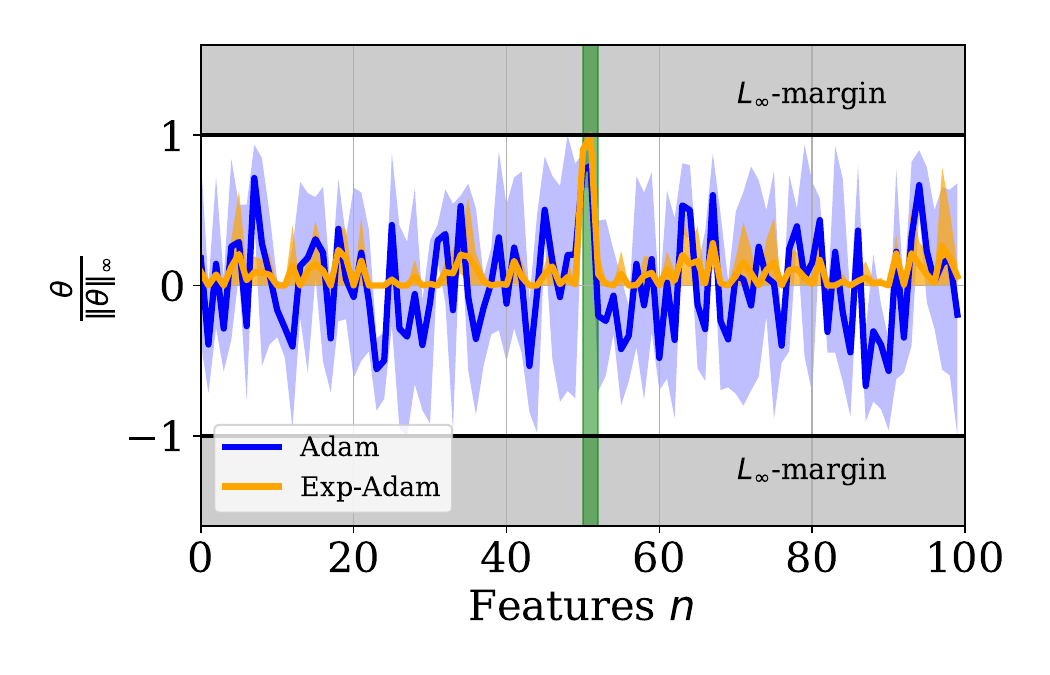}
         \caption{Adam vs Exp-Adam.}
         \label{fig:adam_vs_expAdam_implict_bias}
     \end{subfigure}
     \caption{Implicit bias of additive vs.\ multiplicative steepest descent on sparse linear classification. Each panel plots the final classifier $\theta$, normalized by $||\theta||_{\infty}$ across the $D$ features. The two informative coordinates of the sparse teacher $\theta_{\star}$ are highlighted by the green vertical strip. Each optimizer is averaged across 3 seeds and the shaded bands show the variance across them. The grey strips at $\pm 1$ mark the $L_\infty$-margin region. In (a), SGD (blue) shows slight spikes on the informative coordinates but spreads non-trivial mass across the spurious ones, consistent with its $L_2$ max-margin implicit bias; Exp-SGD (orange), the multiplicative mirror-descent counterpart drives spurious features into a tight band near zero and spikes only on the informative ones, corresponding to known $L_1$ implicit bias. In (b), the same comparison with Adam (blue) and its multiplicative composition Exp-Adam (orange): Adam saturates at $\pm 1$ across nearly all features in line with the $L_\infty$  implicit bias whereas Exp-Adam concentrates mass on the two informative coordinates and drives spurious coordinates near zero, exhibiting an $L_1$ sparsity bias.}
     \label{fig:implicit_bias_toy_example_additive_vs_multiplicative}
\end{figure}

\textbf{Optimization details and hyperparameter selection}: All optimizers are ran in the full-batch regime (batch size $=N$) for $T = 10^4$ epochs across $3$ data seeds. We use $\eta = 10^{-2}$ as the baseline learning rate for SGD and Exp-SGD, and $\eta = 10^{-1}$ for SignSGD; these settings give comparable per-step displacement of $\theta$ since SignSGD's update has unit per-coordinate magnitude regardless of the gradient norm, whereas SGD/Exp-SGD scale with $|\nabla \mathcal{L}|$. Note that on the exponential loss, $|\nabla \mathcal{L}(\theta)| \to 0$ exponentially fast as the margin grows, so any optimizer whose step is proportional to $|\nabla \mathcal{L}|$ including SGD, Exp-SGD, and $\cosh$ has a vanishing effective step size in the last phase, while SignSGD does not. Adam, Exp-Adam, and Adam-Exp use the same $\eta = 10^{-2}$ together with the standard momentum and bias-correction parameters $(\beta_1, \beta_2, \varepsilon) = (0.9, 0.999, 10^{-8})$. For $\cosh$-entropy, the choice of learning is more delicate than for other optimizers, We perform a grid sweep $\eta \in \{0.001, 0.005, 0.009, 0.01, 0.03, 0.05, 0.1, 0.2, 0.3, 0.5, 0.7, 0.9\}$ across $3$ data seeds and select the value that minimizes the final exponential loss $\mathcal{L}(\theta_T)$ while keeping $\eta < 1$ to avoid instability. The best performing setting is $\eta = 0.9$ as shown in \ref{fig:cosh_loss}.

\begin{figure}[h]
     \centering
     \begin{subfigure}[b]{0.49\linewidth}
         \centering
         \includegraphics[width=\linewidth]{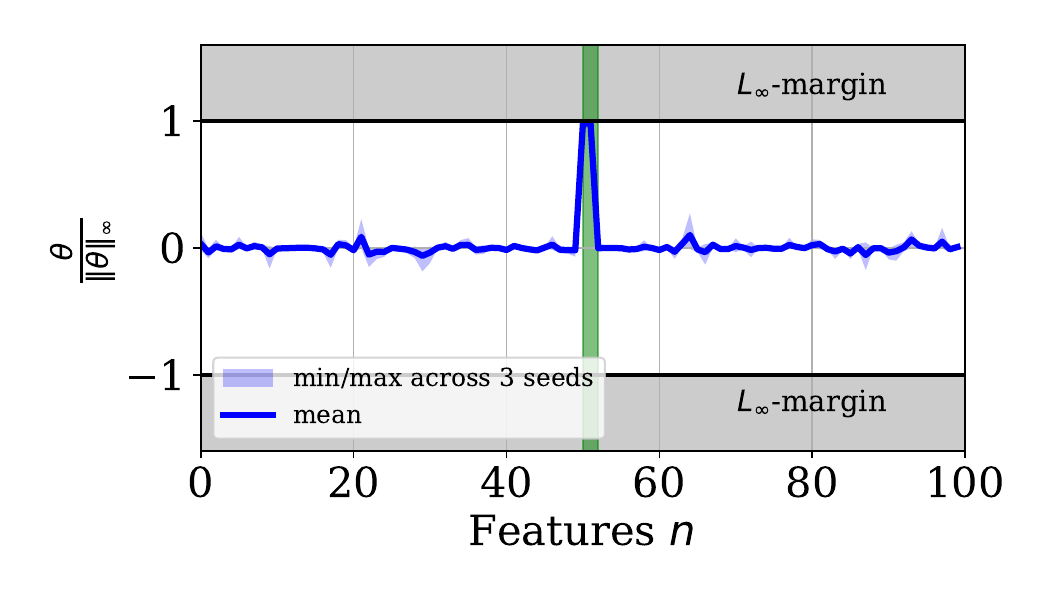}
         \caption{Hyperbolic Entropy.}
         \label{fig:hyp_entropy_MD}
     \end{subfigure}
     \hfill
     \begin{subfigure}[b]{0.49\linewidth}
         \centering
         \includegraphics[width=\linewidth]{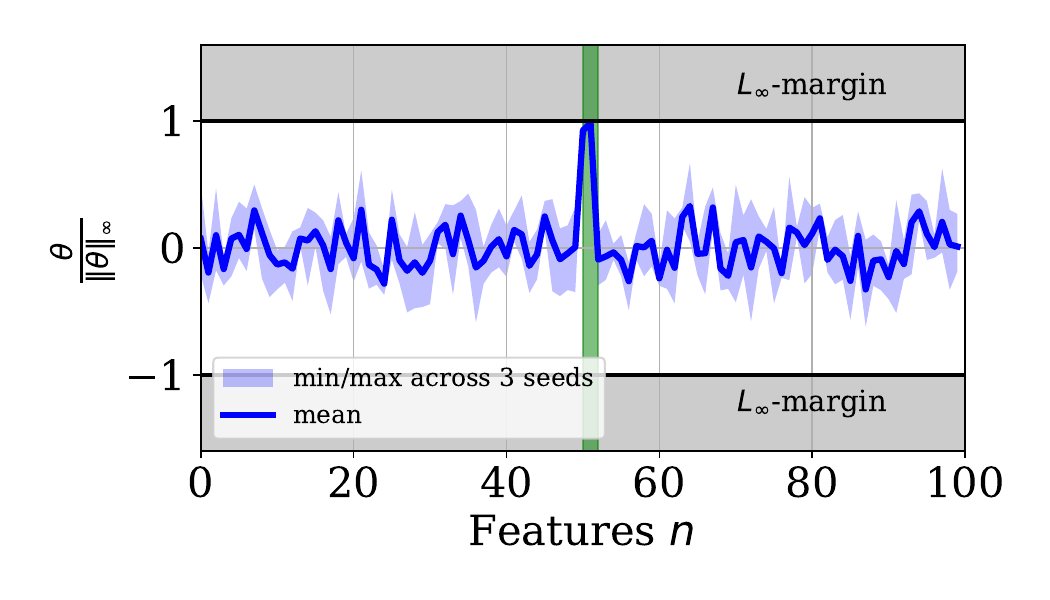}
         \caption{Cosine-Hyperbolic Entropy.}
         \label{fig:cosh_entropy_MD}
     \end{subfigure}
     \caption{Sparse linear classification with mirror maps. (a) the learned features by the hyperbolic entropy, similarly as the exponential update corresponding the entropy mirror map gives rise to an $L_1$ implicit bias. (b) In contrast, the $\cosh$-entropy spreads more mass along the non-informative coordinates.}
     \label{fig:implicit_bias_toy_example_hyperbolic_vs_cosh}
\end{figure}

\begin{figure}[h]
     \centering
     \begin{subfigure}[b]{0.49\linewidth}
         \centering
         \includegraphics[width=\linewidth]{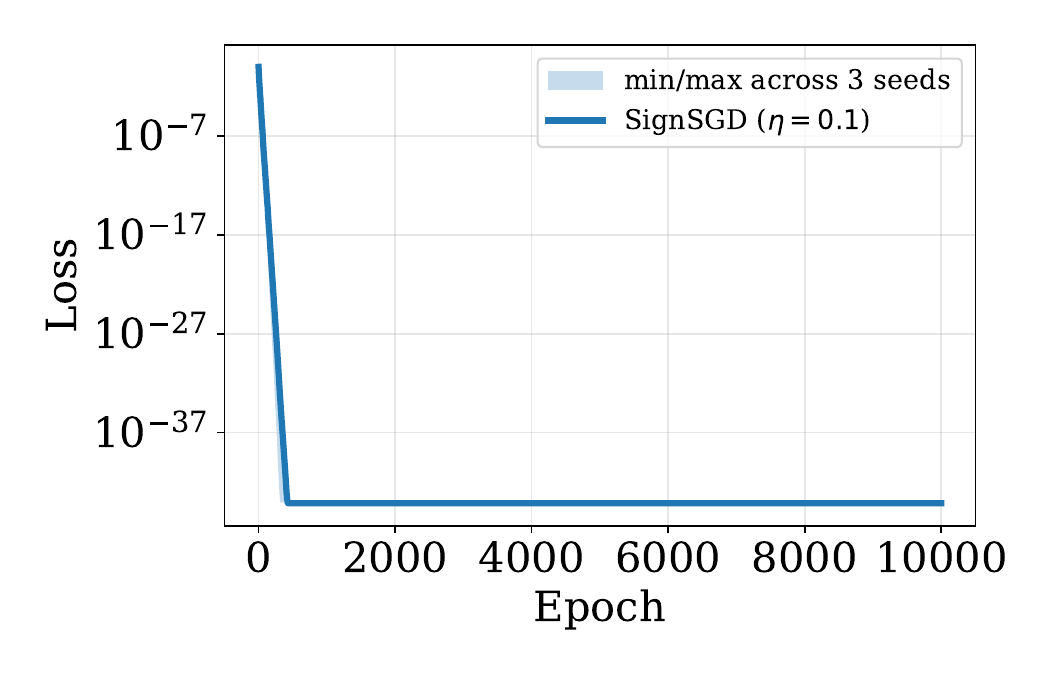}
         \caption{SignSGD.}
         \label{fig:signSGD_loss}
     \end{subfigure}
     \hfill
     \begin{subfigure}[b]{0.49\linewidth}
         \centering
         \includegraphics[width=\linewidth]{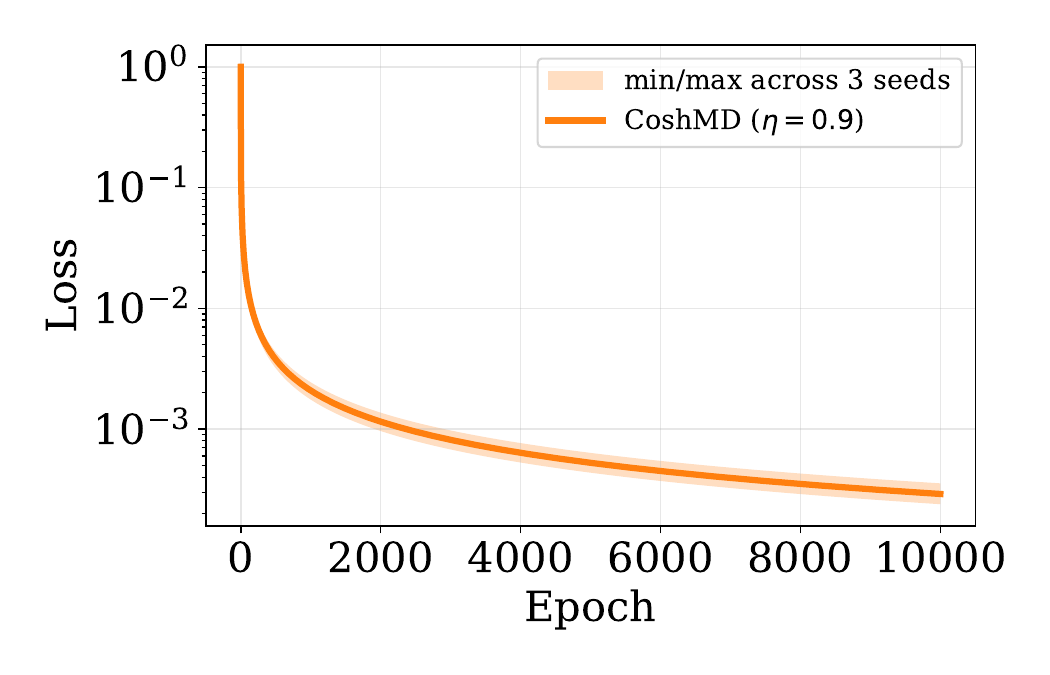}
         \caption{$\cosh$-entropy.}
         \label{fig:cosh_loss}
     \end{subfigure}
     \caption{The evolution of the loss, for signSGD and the $\cosh$-entropy. Observe that $\cosh$-entropy has much slower convergence. }
     \label{fig:toy_example_loss_vs_epochs_signsgd_and_cosh}
\end{figure}

\begin{figure}[h]
         \centering
         \includegraphics[width=0.49\linewidth]{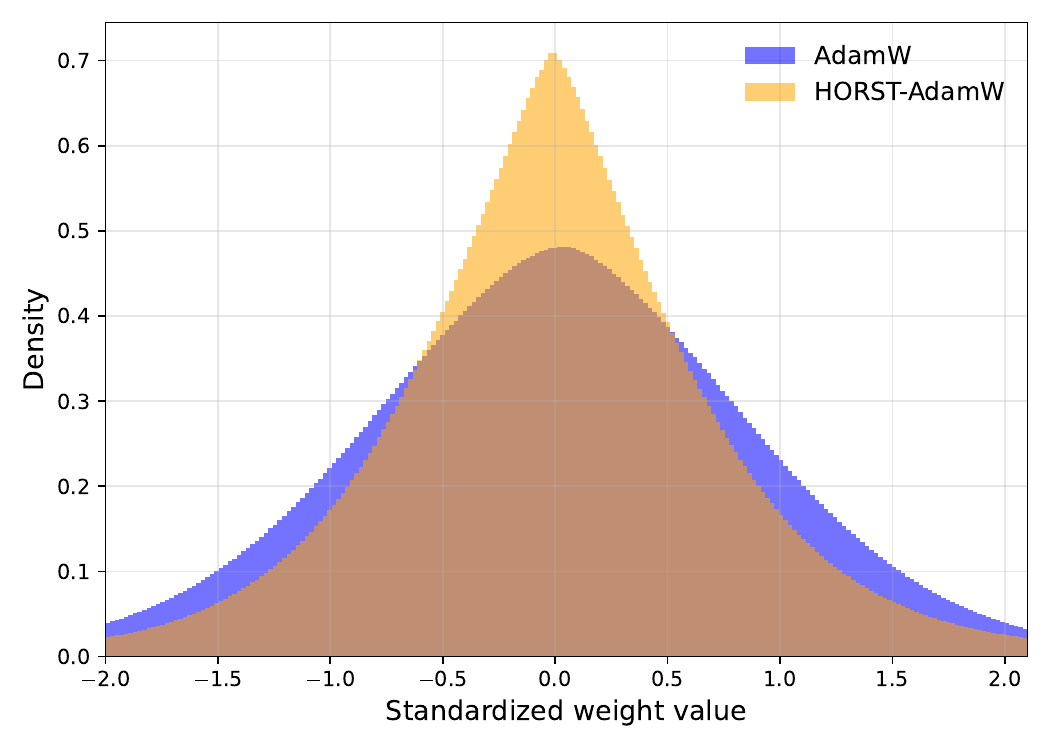}
         \label{fig:image2}
    \caption{\textbf{HORST-AdamW induces a sparser weight distribution.} Standardized weight distributions at end of training for a dense GPT-2 Small model trained on SlimPajama-6B for 25K iterations with HORST-AdamW vs. AdamW. We observe that HORST-AdamW concentrates weights sharply around zero with lighter tails, while AdamW retains a broader, near-Gaussian profile. This indicates the presence of an implicit $L_1$ bias.}
\label{fig:weight_distribution_small_scale_gpt2_on_slimpajama_dense_7layer_vs_12layer}
\end{figure}

\newpage

\section{Details on vision transformer experiment}\label{appendix : details vision}
For the DeiT training setup we use standard data augmentations such as label smoothing, mixup and cutmix. However, we do not use random augmentation. We also use the default hyperparameters of AdamW as given in Table \ref{table : vision transformer hyperparams}.
Moreover, we have ablated over the additional hyperparameters $\alpha \in \{5, 10, 200 \}$ and $\beta \in \{0 , \ 1e-3 , \ 1e-2 \}$ on DeiT-base for $70 \%$ sparsity level and found that $\alpha =5, \beta =0$ works the best. Note $\alpha =200$ was tried first as it was used in \citep{jacobs2025hamhyperbolicstepregulate} for HAM, however we found that $\alpha =200$ is highly unstable and crashes, as this is an effective multiplication of the learning rate. For all other settings we adopt these hyperparameters. Each experimental run was run on $4$ NVIDIA A100 GPUs with 40GB of memory. The code used is based on the repository by \citep{Nelaturu_TurboPrune_High-Speed_Distributed}.
\begin{table}[h]
\caption{Hyperparameters of the sparse training for vision transformers experiments.}\label{table : vision transformer hyperparams}
\centering
\begin{tabular}{lcccccc}
\hline
Model & Learning Rate & Weight Decay & Epochs & Batch Size & Optimizer & Scheduler \\
\hline
DeiT-base & $0.005$ & $1e{-1}$ & $300$ & $1024$ & AdamW & Triangular \\
DeiT-small & $0.005$ & $1e{-1}$ & $300$ & $1024$ & AdamW & Triangular \\
\hline
\end{tabular}
\end{table}

\paragraph{Comparison with HAM}
We also compare with the state-of-the-art sparsity aware optimizer HAM \citep{jacobs2025hamhyperbolicstepregulate}, which is the closest optimizer to HORST. 
We use the default parameter $\alpha =200$. Moreover, as for AdamW, $\beta$ plays the same role as the decoupled weight decay. We consider multiple values $\beta \in \{1e-3, 1e-2, 0\}$ and report the best accuracy values in Table \ref{table : horst ham comparison}. Note that this level of improvement is in line with the reported gains in \citep{jacobs2025hamhyperbolicstepregulate} in Table 8 and 9.
\begin{table}
\centering
    \caption{Comparison with the method HAM on a DeiT-small at $90\%$ sparsity trained on ImageNet for seed $1$. The geometric mismatch between gradient based exponential and the AdamW update makes it harder for HAM to improve over the baseline.}\label{table : horst ham comparison}
\begin{tabular}{c|c}
    \hline
    Method & Val. Acc. \\
    \hline
    HAM & $60.06$ \\
    HORST & $\textbf{70.5}$ \\
    Baseline & $59.23$ \\
    \hline
\end{tabular}
\end{table}

\paragraph{Long run}
To emphasize that HORST's benefits do not become negligible with longer training we provide a single seed long run ($600$ Epochs) for $70\%$ sparsity. We report the validation accuracy in Table \ref{tab:adamw_horst long}. We observe that both optimizers have now improved accuracy and furthermore the gap between them has decreased a little, however, it remains significant. 

\begin{table}[ht]
\caption{Performance comparison of AdamW and HORST-AdamW together with AC/DC on a DeiT-small for $70\%$ sparsity for longer training time of $600$ epochs.}\label{tab:adamw_horst long}
\centering
\begin{tabular}{lc}\hline Optimizer & Value \\
\hline
AdamW        & $73.76$ \\
HORST-AdamW  & $\mathbf{78.96}$ \\\hline
\end{tabular}
\end{table}

\newpage

\section{Details on GPT-2 experiment}\label{appendix : details language}

We train a GPT-2 Small model ($n_{\text{layer}} = 12$, \texttt{n\_head}$=12$, \texttt{n\_embd}$=768$, context length $512$,
$\sim 124$M parameters with tied input/output embeddings) on the SlimPajama-6B subset using the GPT-2 BPE tokenizer (vocabulary size $50304$). All runs share the hyperparameters in Table~\ref{table:gpt2_hyperparams}, use $3000$ warmup steps, run in \texttt{bfloat16} precision, and use $\beta_1=0.9$, $\beta_2=0.95$ for AdamW. For the HORST-AdamW variant we additionally set $\alpha = 5$ and $\beta = 0$. Experiments were conducted using a single NVIDIA A100 GPU with 80 GB of memory. The code used is based mainly on the repository by \cite{semenov2025benchmarking}.
\begin{table}[ht]
\caption{Hyperparameters of the dense pretraining for GPT-2 Small experiments. AdamW and HORST-AdamW use the same configuration.}
\label{table:gpt2_hyperparams}
\centering
\begin{tabular}{lccccc}
\hline
Model & Learning Rate & Weight Decay & Iterations & Batch Size & Scheduler \\
\hline
GPT-2 Small & $1\mathrm{e}{-3}$ & $1\mathrm{e}{-1}$ & $25,000$ & $128$ & Cosine \\
\hline
\end{tabular}
\end{table}

For the one-shot pruning experiments we apply per-tensor (layerwise) unstructured magnitude pruning at sparsity levels $s \in \{0.1, 0.2, 0.3, 0.4, 0.5, 0.6\}$ to the dense AdamW and HORST-AdamW checkpoints. For each prunable weight tensor $W$ and target sparsity $s$ we zero the $\lfloor s \cdot |W| \rfloor$ entries of smallest absolute magnitude and leave the rest untouched, computing the threshold independently per tensor. Pruning is applied only to the \texttt{nn.Linear} weight tensors inside transformer blocks $1, \ldots, n_{\text{layer}} - 2$ (four tensors per block: the packed $Q/K/V$ projection, the attention output projection, and the two MLP projections). The token and positional embeddings, the tied LM head, all LayerNorm parameters and biases, and the first \& last transformer blocks are kept dense throughout. No fine-tuning, weight recovery, or calibration data is used. Validation perplexity and
accuracy are evaluated on the held-out SlimPajama validation split using the same evaluation pipeline used during training.
\begin{table}[h]
\caption{Validation perplexity (lower is better) on SlimPajama-6B after one-shot layerwise unstructured magnitude pruning of dense GPT-2 Small checkpoints; no fine-tuning. HORST-AdamW achieves the lowest perplexity at every non-trivial sparsity level, with the gap to baseline AdamW significantly widening at higher sparsities.}
\label{tab:gpt2_small_one_shot_pruning_experiment_val_perplexity_table}
\centering
\begin{tabular}{lcccccc}
\hline
Optimizer & $\text{Dense}$ & $10\%$ & $20\%$ & $30\%$ & $40\%$ & $50\%$ \\
\hline
AdamW         & $23.46$ & $23.61$ & $24.54$ & $33.49$ & $86.50$  & $237.17$ \\
HAM           & $23.48$ & $23.60$ & $24.31$ & $29.58$ & $54.58$  & $105.76$ \\
HORST-AdamW   & $23.61$ & $23.70$ & $\mathbf{24.32}$ & $\mathbf{28.78}$ & $\mathbf{44.16}$ & $\mathbf{69.58}$ \\
\hline
\end{tabular}
\end{table}

\newpage

\end{document}